\documentclass{article}
\usepackage{ijcai26}

\usepackage{times}
\usepackage{soul}
\usepackage{url}
\usepackage[hidelinks]{hyperref}
\usepackage[utf8]{inputenc}
\usepackage[small]{caption}
\usepackage{graphicx}
\usepackage{amsmath}
\usepackage{amsthm}
\usepackage{booktabs}
\usepackage{algorithm}
\usepackage{algorithmic}
\usepackage[switch]{lineno}

\usepackage{amssymb,amsfonts}
\newtheorem{definition}{Definition}
\usepackage{subfigure}        
\usepackage{subcaption}
\usepackage{svg}
\usepackage{graphicx}
\usepackage{todonotes}
\usepackage{textcomp}
\usepackage{xcolor}
\def\BibTeX{{\rm B\kern-.05em{\sc i\kern-.025em b}\kern-.08em
    T\kern-.1667em\lower.7ex\hbox{E}\kern-.125emX}}

\newcommand\Un{\mathcal{U}}
\newcommand\F{
    \protect\tikz[baseline]{
    \protect\path[draw,line width=.12ex,line join=round]
      (0ex,.6ex) -- (.95ex,1.55ex) -- (1.9ex,.6ex) -- (.95ex,-.35ex) -- cycle;
}}
\newcommand\G{
    \protect\tikz[baseline]{
    \protect\path[draw, line width=.12ex, line join=round]
      (0ex,-.2ex) -- (0ex,1.3ex) -- (1.5ex,1.3ex) -- (1.5ex,.-.2ex) -- cycle;
}}

\newcommand\Leaf{\ensuremath\mathsf{Leaf}}
\newcommand\Fallback{\ensuremath\mathsf{Fback}}
\newcommand\Sequence{\ensuremath\mathsf{Seq}}
\newcommand\Parallel{\ensuremath\mathsf{Par}}

\newcommand\X{\bigcirc}
\newcommand\tree{\mathcal{T}}

\newcommand\true{\mathit{true}}
\newcommand\false{\mathit{false}}
\newcommand\commentout[1]{}

\usepackage{tikz}
\usetikzlibrary{positioning,arrows,petri,fit,backgrounds}

\tikzset{
    place/.style={
        circle,
        thick,
        minimum size=6mm,
        draw
    },
    transitionV/.style={
        rectangle,
        thick,
        fill=black,
        minimum height=6mm,
        inner xsep=1pt
    },
    transitionH/.style={
        rectangle, 
        thick,
        fill=black,
        minimum width=6mm, 
        inner ysep=1pt
    },
    transitionD/.style={
        rectangle, 
        thick,
        fill=black,
        minimum width=6mm, 
        inner ysep=1pt,
        rotate=45,
    },
    transitionD2/.style={
        rectangle, 
        thick,
        fill=black,
        minimum width=6mm, 
        inner ysep=1pt,
        rotate=-45,
    },
}
\tikzset{
  weight/.style={
    pos=0.5,
    sloped,
    above,
    overlay,
    fill=white,
    inner sep=1pt,
    font=\small
  }
}
\tikzset{
  weight2/.style={
    pos=0.5,
    above right=-0.01cm,
    overlay,
    fill=white,
    inner sep=1pt,
    font=\small
  }
}

\newtheorem{example}{Example}
\newtheorem{theorem}{Theorem}

\title{A Reward-Petri-Net Interpretation of Temporal Behavior Trees}

\author{
Till Schmeil$^{1,2}$ \and
Günther Waxenegger-Wilfing$^{1,2}$ \and
Sebastian Schirmer$^2$
\affiliations
$^1$University of Würzburg, Germany\\
$^2$German Aerospace Center (DLR), Germany\\
\emails
till.schmeil@uni-wuerzburg.com
}

\begin{document}

\maketitle

\begin{abstract}
This paper introduces an interpretation of Temporal Behavior Trees (TBTs) as Reward-Petri-Nets (RPNs) for reinforcement learning (RL). Designing reward functions for complex, long-horizon robotic tasks is notoriously difficult, especially when tasks have hierarchical structure and temporal constraints. TBTs extend conventional behavior trees (BTs) used in robotic applications by incorporating temporal properties into their leaf nodes.
This allows TBTs to represents not only the behavioral task structure defined by BT operators such as Sequence, Fallback, and Parallel, but also the task's temporal constraints. In this work, the constraints are specified in the leaf nodes using Linear Temporal Logic. In order to inform RL rewards using TBTs, we provide a translation from TBT into a Petri Net (PN) and show how rewards can be automatically assigned based on the TBT's structure, resulting in a RPN. In a series of increasingly challenging environments, we demonstrate how TBT-based rewards enable learning where vanilla RL fails, improve sample efficiency, and offer flexible, intuitive control over the learning progress. We showcase the learning impact by using different reward distribution schemes and TBT structures.
\end{abstract}

\section{Introduction}\label{sec:introduction}
Long-horizon sparse reward environments remain a significant challenge for reinforcement learning (RL) methods.
We study the problem of generating a dense reward function using Temporal Behavior Tree (TBT) specifications.
Reward functions must describe the agent's objective in an RL-compatible form, which may be inherently impossible for certain temporal properties, such as task sequences or alternating repetitions \cite{balakrishnan2023symbolic}.
As a consequence, rewards often provide insufficient guidance to an agent toward desirable and safe long-term behavior.

To address these shortcomings, specification languages have been investigated to provide rewards.
The key idea is to automatically derive a reward function from a formal specification that captures the task.
For example in \cite{abate2020temporal}, specifications were translated into automata whose transitions emit rewards.
A major advantage of this approach is that rewards are automatically and systematically distributed across large state spaces.
However, this automation comes at the cost of reduced user control, particularly when certain states or sub-tasks are more critical than others for accomplishing the overall objective.
In practice, users often possess intuitive knowledge about which sub-tasks are key and how their relative importance should be weighted in achieving a long-horizon task.

In this work, we use Temporal Behavior Trees (TBTs) as the specification formalism for reward generation. 
TBTs are inspired by behavior trees.
A Behavior tree (BT) specifies complex sequences of actions for applications in robotics and cyber-physical systems (CPS) using operators for sequential composition, fallbacks, repetitions, and parallel executions.
TBTs combine BTs with temporal formulas in their leaf nodes for monitoring purposes.
An important advantage of TBTs is their flexibility and modularity. 

\begin{figure}[b]
    \centering
    \begin{tikzpicture}
        \node[inner sep=0pt] (img) at (0,0)
            {\includegraphics[width=0.4\linewidth]{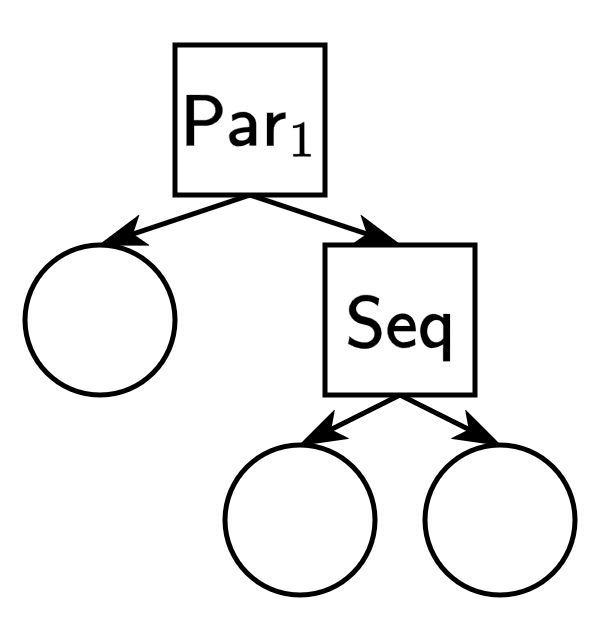}};
        \node at (-1.15,-0.0) {$\F a$};
        \node at (0,-1.2) {$\F b$};
        \node at (1.15,-1.2) {$\G\ b$};
        
    \end{tikzpicture}
    \caption{A TBT specification $\Parallel_1([\F a, \Sequence(\F b, \G b)])$ where the agent should stay in $a$ or eventually reach $b$ and then remain there.}
    \label{fig:exampleintro}
\end{figure}

Figure \ref{fig:exampleintro} illustrates an example TBT specification. 
The root is a parallel node, annotated with the subscript (1), indicating that satisfaction of at least one of its subtrees is sufficient.
The left subtree consists of the temporal logic formula $\F a$, which specifies that proposition $a$ must eventually hold. 
The right subtree is a sequence node that evaluates its children from left to right: first, the agent must eventually satisfy $b$, and subsequently always maintain $b$. 
A corresponding reward interpretation of the TBT could emit $1$ upon satisfying $\F a$, and rewards of $0.75$ and $0.25$ for satisfying $\F b$ and $\G b$, respectively.
This allocation reflects the user's intuition that the parallel branches are equivalent (in accumulated reward) and that achieving $b$ is more important than maintaining it thereafter.


In general, successful or violating leaf nodes serve as landmarks for learning the overall TBT specification with control flow nodes (non-leaf) used to direct the agent's attention towards active leaf nodes.
Upon completion of a node, a reward is emitted and new nodes become active according to the TBT's semantics.
%
%
We formalize this by introducing a construction of Reward Petri Nets (RPNs) derived from TBT specifications.
RPNs extend classical Petri nets with guard predicates, rewards, and actions associated with transition firings.
We show that RPNs naturally encode TBT operators and provide a mechanism for deriving reward functions.
Experimental results demonstrate that by including an RPN in the RL formulation allows learning hard exploration tasks where vanilla RL fails. Further, TBT-based reward functions scale efficiently with increasing environment complexity.

The remainder of the paper is structured as followed: Section \ref{sec:preliminaries} provides background on specification-guided RL, Section \ref{sec:core} provides the construction from TBT to RPN, finally Section \ref{sec:experiments} demonstrates experimental results.

\section{Related Work}\label{sec:relatedwork}

Long-horizon reinforcement learning presents difficulties in effective credit assignment and exploration due to delayed and sparse rewards. Reward shaping approaches provide intermediate guidance signals to accelerate learning. \cite{ng1999policy} formalized potential-based reward shaping and showed that such shaping preserves optimal policies while improving sample efficiency. Subsequent work has explored automated shaping via learned potentials or auxiliary objectives, including methods that derive shaping rewards from representation learning or progress-based metrics \cite{jaderberg2017reinforcement,arjona2019rudder}.

While reward shaping attempts to improve credit assignment directly, exploration methods instead focus on increasing the probability of encountering informative rewards. Count-based methods encourage agents to visit novel states by approximating state visitation counts in high-dimensional spaces \cite{bellemare2016unifying,tang2017exploration}. Curiosity-based and intrinsic motivation methods reward agents for prediction errors or information gain, promoting exploration of states where dynamics are uncertain \cite{schmidhuber1991possibility,pathak2017curiosity,intrinsic}.

Rather than modifying learning signals or exploration behavior, hierarchical approaches reduce the effective horizon through temporal abstraction. The options framework introduced higher-level actions that operate across multiple time steps \cite{sutton1999between,machado2023}, while later work developed techniques for automatically discovering useful skills \cite{bacon2017option,klissarov2023}. Although these methods can significantly improve learning efficiency, reliably discovering generalizable and transferable abstractions remains challenging.

Often, desirable behavior is known and can be specified using formal methods. Popular examples include various Automata \cite{camacho2021rewardmachine,icarte2022reward,balakrishnan2023symbolic} and Temporal Logics (TL) \cite{li2017robot,icarte2019ltl,abate2020temporal,siqi2024tractable,kantaros2024} with an overlap, where TL formulas are translated into automata. The specification then extends the RL formulation to assign rewards, or guide exploration.

TBTs offer the advantage of strictly higher expressivity than the common Linear TL (LTL) and Signal TL (STL). Furthermore, Petri nets provide a natural representation of TBTs, which prevent state explosion for parallel tasks and permit fine control over dense reward functions. In addition, TBTs facilitate modular construction and human-readability.

\section{Preliminaries}\label{sec:preliminaries}
\subsection{Reinforcement Learning}\label{subsec:rl}

Reinforcement Learning methods enable policy synthesis for optimal decision-making.
An agent picks actions according to its policy, then observes the resulting next state of its environment and a reward that evaluates the chosen action.
Formally, this interaction is often modelled as a Markov Decision Process (MDP) \cite{sutton2018reinforcement}.

\begin{definition}[Markov Decision Process]
    A MDP is a tuple $\mathcal{M} = (S, A, P, R, \rho_0)$ with state space $S$, action space $A$, (partial) transition probability distribution $P: S \times A \times S \rightarrow [0, 1]$ with $P(s, a, s') = \Pr(s' \mid s, a)$, reward function $R(s, a, s')$ that maps a transition ($S \times A \times S) $ to $\mathbb{R}$, and initial state distribution $\rho_0 \subset S$, where $s, s' \in S$ and $a \in A$.
\end{definition}

The agent's objective is to learn an optimal policy $\pi^*$ that maximizes the (discounted) total reward, i.e., $\max \Sigma_{t=0}^{\infty} \gamma^tr_t$ where $r_t$ is the given reward at time $t$ and $\gamma \in [0,1]$ is a discount factor.
When the reward function is part of a finite state machine that yields rewards based on the current state, it is referred to as reward machine \cite{icarte2022reward}.
The key idea is to use atomic propositions to abstract environmental information into high-level events, e.g., a position in a grid world is translated into the proposition $\mathit{atCoffeePlace}$.
Let $\mathit{AP}$ be a finite set of atomic propositions and $\Sigma = 2^{\mathit{AP}}$ be a finite alphabet.
Further, we use a labeling function $L: S \rightarrow \Sigma$ that maps states to atomic propositions.




\subsection{Temporal Logic}\label{subsec:tl}
LTL \cite{4567924} formulas have the following syntax:
\[
\varphi := \true \mid p \in \mathit{AP} \mid \neg \varphi \mid \varphi \lor \varphi \mid \X \varphi \mid \varphi\ \Un\ \varphi
\]
As abbreviations, we introduce the \emph{Eventually}-operator $\F \varphi := \true\ \Un\ \varphi$, the \emph{Globally}-operator $\G \varphi := \neg (\F \neg \varphi)$.
Further, we define the bounded operators $\F_{[a,b]} \varphi$ and $\G_{[a,b]} \varphi$ with $a,b \in \mathbb{N}$ by $\bigvee_{i=a}^b \X^i \varphi$ and $\bigwedge_{i=a}^b \X^i \varphi$, recursively.

The semantics of LTL is interpreted over infinite words $w=a_0a_1\dots \in \Sigma^\omega$.

\begin{definition}[LTL semantics]
Let $w$ be an infinite word and $i \in \mathbb{N}$ be a position, the semantics of an LTL formula $\varphi$ is defined as follows:
\[
\def\arraystretch{1.2} \begin{array}{lcl}
 \sigma, i \models \true\ \\
 \sigma, i \models p & \iff & p \in a_i\\
 \sigma, i \models \neg \varphi & \iff & \sigma,i \not\models \varphi\\
 \sigma, i \models \varphi_1 \lor \varphi_2 & \iff & \sigma,i \models \varphi_1 \ \text{or}\ \sigma,i \models \varphi_2\\
 \sigma, i \models \X \varphi & \iff & \sigma,i+1 \models \varphi\\
 \sigma, i \models \varphi_1\ \Un\ \varphi_2 & \iff & \exists k \geq i\ \text{with}\ \sigma, k \models \varphi_2\ \text{and}\\
 & & \ \forall l \in \{i, \dots, k-1\}.\ \sigma, l \models \varphi_1
 \end{array} 
\]
\end{definition}


Let $u=a_0,a_1,\dots,a_n \in \Sigma^*$ be a finite word and let $uw$ be a infinite word where a finite word $u$ is concatenated by an infinite word $w$. 
We use LTL$_3$ \cite{10.1145/2000799.2000800} for interpreting finite words using a three-valued semantics with value set $\mathbb{B}_3 = \{ \true, \false, ? \}$. 

\begin{definition}[LTL$_3$ semantics]\label{def:ltl3}
Let $u$ be a finite word, the semantics of LTL$_3$ is defined as follows:
\[
u \models \varphi :=
\begin{cases}
  \true & \text{if } \forall \sigma \in \Sigma^\omega : u\sigma, 0 \models \varphi,\\
  \false & \text{if } \forall \sigma \in \Sigma^\omega : u\sigma, 0 \not\models \varphi,\\
  ? & \text{otherwise.}
\end{cases}
\]
\end{definition}

LTL$_3$ can be transformed to a finite state machine using the construction presented in \cite{10.1145/2000799.2000800}.
Further, by Definition \ref{def:ltl3}, once a state produces $\true$ and $\false$ it remains producing the same verdict, similar to a sink state.

\subsection{Temporal Behavior Tree}\label{subsec:tbt} 
A TBT \cite{10.1145/3641513.3650180} borrows the structural operators of BTs, but augments leaf nodes with monitoring properties

\begin{definition}[Syntax of Temporal Behavior Trees]
    Let $\varphi$ be an LTL formula, we construct a temporal behavior tree using the following syntax:
     \[\def\arraystretch{1.2} \begin{array}{rlll}
    \tree & := \Leaf(\varphi), & &\leftarrow \text{Leaf node }\\\ 
    & |\ \Fallback([\tree, \dots, \tree]), & &\leftarrow \text{Fallback node }\\\
    & |\ \Parallel_M( [\tree, \dots, \tree]), & M \in \mathbb{N} &\leftarrow \text{Parallel node }\\\
    & |\ \Sequence( [\tree, \cdots, \tree]), & &\leftarrow \text{Sequence node }\\\
    \end{array}\]
\end{definition}

Informally, an LTL formula at the leaf node specifies that the word must satisfy the formula.
Likewise, $\Fallback([\tree_1, \ldots, \tree_k])$ mimics the semantics of a ``fallback'' node in a behavior tree: at least one of the subtrees $\tree_1, \ldots, \tree_k$ must eventually be satisfied by the trace. 
$\Parallel_M([\tree_1, \dots, \tree_k])$ denotes a parallel operator that specifies that at least $M$ distinct subtrees must be satisfied simultaneously by the trace $\sigma$. 
$\Sequence([\tree_1, \cdots, \tree_k])$ is a sequential node that denotes that its children must be satisfied in sequence.
We denote $\sigma \models \tree$ as the trace satisfying the TBT specification.

We introduce a parameterized syntactic shortcut for TBT operators.
Let $x$ range over a finite set $X$. 
We write $\langle x \in X \rangle$ to denote a parameter that is syntactically substituted into all symbols occurring in the following subtrees.
This notation expands to a replication of the subtree for every $x \in X$.
For instance, instead of writing $\Fallback([\Sequence([\F \mathit{redDoor}, \F \mathit{redDoorOpen}]), \allowbreak\Sequence(\allowbreak[\F \mathit{greenDoor}, \F \mathit{greenDoorOpen}])])$, we write $\Fallback\langle\allowbreak x \in \{red, green\}\rangle ([\allowbreak\Sequence([\allowbreak\F \mathit{xDoor}, \F \mathit{xDoorOpen}])])$.

\subsection{Petri Nets}
A Petri net is a mathematical model for describing concurrent, asynchronous, and distributed processes using places, transitions, and tokens \cite{petri1962kommunikation}.
In this paper, places represent TBT nodes, guarded transitions model the satisfaction of a TBT node, and tokens track the progress toward satisfying the TBT.
Let a multiset $M$ over a set $S$ be a function $M: S \rightarrow \mathbb{N}$.

\begin{definition}[Petri Net]
A Petri net is a tuple $\mathcal{N} = (P, T, F, M_0)$ with disjoint finite sets of places $P$ and of transitions $T$, the flow relation $F$ be a multiset over $(P \times T) \cup (T \times P)$, and the initial marking $M_0$ be a multiset over $P$.
\end{definition}

Intuitively, a Petri net defines the flow of tokens from places over transition to other places.
A marking $M$ is a multiset over $P$ and represents the current state of the net, i.e., the number of tokens in place $p$.
A transitions $t \in T$ is firing at marking $M$ if $\forall p \in P, M(p) \ge  F(p, t)$, i.e., input places must have sufficient tokens.
Upon firing a transition $t$, the state of the Petri net changes according to $F$ by consuming tokens from input places and producing tokens to output places, i.e., $\forall p \in P, M'(p) = M(p) - F(p, t) + F(t, p)$.
The Petri net stops when it reaches a dead marking where no transition can fire.
A Petri net is deterministic, if for every reachable marking there is only one transition firing.

\begin{example}
Consider the Petri net shown in Figure \ref{fig:dpn}.
It consists of two transitions $T=\{t_0,t_1\}$ and three places $P=\{q_0,q_1,q_2\}$, with an initial marking in which $q_0$ and $q_2$ contain one token each.
The net is deterministic.
From the initial marking, only $t_0$ is enabled, since $t_1$ requires two tokens in $q_0$.
After firing $t_0$, one token is consumed from $q_0$ and $q_2$ and two tokens are produced in $q_0$.
Thereafter, $t_0$ is disabled, while enough tokens are in $q_0$ to enable $t_1$.
Firing $t_1$ reaches a dead marking as only one token is in $q_1$ from where no transitions are possible.
Note that if $q_2$ was removed, the net would no longer be deterministic, since $t_0$ could then be fired indefinitely, producing an unbounded number of tokens.
\end{example}

\begin{figure}[h!]
    \centering
    \begin{tikzpicture}[node distance=1cm and 2cm,>=stealth',bend angle=45,thick]
    \node [transitionV,label=above:$t_0$] (t0) {};
    \node [transitionV,label=above:$t_1$] (t1) [below right=1cm and 1.5cm of t0] {};
    \node [place,tokens=1,label=below:$q_0$] (q0) [below= of t0] {}
        edge[pre,bend right=70] node[weight] {$2$} (t0)
        edge[post,bend left=70] node[weight] {$1$} (t0)
        edge[post] node[weight] {$2$} (t1);
    \node [place,tokens=1,label=below:$q_2$] (q2) [right= 1.5cm of t0] {}
        edge[post] node[weight] {$1$} (t0);
    \node [place,tokens=0,label=below:$q_1$] (q1) [right= 1.5cm of t1] {}
        edge[pre] node[weight] {$1$} (t1);
    \end{tikzpicture}
    \caption{Deterministic Petri net}
    \label{fig:dpn}
\end{figure}
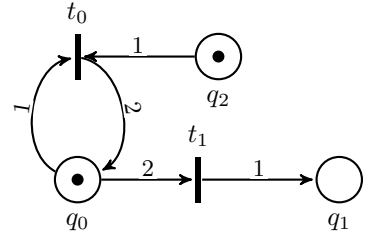
\section{Building Reward Petri Nets}\label{sec:core}
In this section, we extend Petri nets to Reward Petri nets (RPN), present the construction of (RPN) from TBT specification, and show how rewards are assigned given distribution functions.

RPNs extend Petri nets by outputting rewards and actions upon firing transitions and adding guards to transitions.

\begin{definition}[Reward Petri Net]
A Reward Petri net is a tuple $\mathcal{RN} = (P, T, F, M_0, R, A, \mathit{Act}, V, G)$ which extends a Petri net by rewards, actions, and guards. 
The function $R: T \rightarrow \mathbb{R}$ assigns real-values rewards to each transition, the function $\mathit{Act}: T \rightarrow A$ assigns actions from the action set $A$ to transitions, $V$ is a set of variables used in the guard expressions, and $G: T \rightarrow B(V)$ assigns to each transition a Boolean-valued guard expression over the variables $V$. 
Here, $B(V)$ denotes the set of Boolean expressions formed from the variables in $V$.
\end{definition}

The RPN progresses in the same way as standard Petri net, except that the firing rule for transitions is extended.
A transition $t$ fires if two conditions are satisfied: first, its input places contain the required number of tokens, and second, its guard $G(t)$ evaluates to true.

\subsection{Constructing Reward Petri Nets from TBTs}
We begin by constructing the places $P$, the transitions $T$, the flow relation $F$, and the initial marking $M_0$.
A given TBT specification is recursively translated into a Petri net $N=(P,T, F, M_0)$, in which recursive calls are represented as dotted circles with the corresponding parameter inside:

\begin{description}
\item[$\bullet~\Leaf(\varphi)$:] ~\\
    \begin{tikzpicture}[node distance=1cm and 2cm,>=stealth',bend angle=45,thick]
        \node [place,label=below right:$p_{?\varphi}$] (q0) {};
        \node [transitionV,label=above:$t_{\varphi}$] (t0) [right= 1cm of q0] {};
        \node [place,label=below:$p_\varphi$] (q1) [right= 1cm of t0] {}; 
        \node [transitionD2,label=below left:$t_{\neg\varphi}$] (t1) [above left= 1cm of q0] {};
        \draw[->] (q0) edge node[weight] {$1$} (t0);
        \draw[->] (t0) edge node[weight] {$1$} (q1);
        \draw[->,bend left=70] (t1) edge node[weight2] {$1$} (q0);
        \draw[->,bend left=70] (q0) edge node[weight2] {$1$} (t1);
    \end{tikzpicture}

\item[$\bullet~\Fallback(\lbrack\tree_1, \dots, \tree_n\rbrack)$:]~\\
    \begin{tikzpicture}[node distance=1cm and 2cm,>=stealth',bend angle=45,thick]
        \node [dotted,place,label=center:$\tree_1$] (q0) [] {};
        \node [dotted,place,label=center:$\dots$] (q1) [below= 0.5cm of q0] {};
        \node [dotted,place,label=center:$\tree_n$] (q2) [below= 0.5cm of q1] {};
        \node [transitionV,label=above:$t_{\tree_1}$] (t0) [right= 1cm of q0] {};
        \node [place,label=right:$c_\tree$] (q3) [below= 0.5cm of t0] {};
        \node [transitionV,label=below:$t_{\tree_n}$] (t1) [below= 0.5cm of q3] {};
        \node [place,label=below:$p_{\mathit{sync}}$] (q4) [right= 1cm of q3] {};
        \node [transitionV,label=below:$t_{\tree}$] (t2) [right= 1cm of q4] {};
        \node [place,label=below:$p_\tree$] (q5) [right= 1cm of t2] {};
        
        \draw[->] (q0) edge node[weight] {$1$} (t0);
        \draw[->] (q2) edge node[weight] {$1$} (t1);
        \draw[->,bend left=40] (q3) edge node[weight] {$1$} (t0);
        \draw[->,bend right=40] (q3) edge node[weight] {$1$} (t1);
        \draw[->] (t0) edge node[weight] {$1$} (q4);
        \draw[->] (t1) edge node[weight] {$1$} (q4);
        \draw[->] (q4) edge node[weight] {$1$} (t2);
        \draw[->] (t2) edge node[weight] {$1$} (q5);
    \end{tikzpicture}

\item[$\bullet~\Parallel_M(\lbrack\tree_1, \dots, \tree_n\rbrack)$:]~\\
    \begin{tikzpicture}[node distance=1cm and 2cm,>=stealth',bend angle=45,thick]
        \node [dotted,place,label=center:$\tree_1$] (q0) [] {};
        \node [dotted,place,label=center:$\dots$] (q1) [below= 0.5cm of q0] {};
        \node [dotted,place,label=center:$\tree_n$] (q2) [below= 0.5cm of q1] {};
        \node [transitionV,label=above:$t_{\tree_1}$] (t0) [right= 1cm of q0] {};
        \node [transitionV,label=below:$t_{\tree_n}$] (t1) [right= 1cm of q2] {};
        \node [place,label=$p_{\mathit{sync}}$] (q3) [right= 2cm of q1] {};
        \node [transitionV,label=below:$t_{\tree}$] (t2) [right= 1cm of q3] {};
        \node [place,label=$c_{\tree}$] (q4) [above= 0.5cm of t2] {};
        \node [place,label=$p_{\tree}$] (q5) [right= 1cm of t2] {};
        
        \draw[->] (q0) edge node[weight] {$1$} (t0);
        \draw[->] (q2) edge node[weight] {$1$} (t1);
        \draw[->] (t0) edge node[weight] {$1$} (q3);
        \draw[->] (t1) edge node[weight] {$1$} (q3);
        \draw[->] (q3) edge node[weight] {$M$} (t2);
        \draw[->] (t2) edge node[weight] {$1$} (q5);
        \draw[->,bend right=40] (q4) edge node[weight] {$1$} (t2);
    \end{tikzpicture}

\item[$\bullet~\Sequence(\lbrack\tree_1, \dots, \tree_n\rbrack)$:]~\\
     \begin{tikzpicture}[node distance=1cm and 2cm,>=stealth',bend angle=45,thick]
        \node [dotted,place,label=center:$\tree_1$] (q0) [] {};
        \node [transitionV,label=above:$t_{\tree_1}$] (t0) [right= 0.8cm of q0] {};
        \node [dotted,place,label=center:$\dots$] (q1) [right= 0.8cm of t0] {};
        \node [transitionV,label=above:$t_{\tree_{n-1}}$] (t1) [right= 0.8cm of q1] {};
        \node [dotted,place,label=center:$\tree_n$] (q2) [right= 0.8cm of t1] {};
        \node [transitionV,label=above:$t_{\tree}$] (t2) [right= 0.8cm of q2] {};
        \node [place,label=below:$p_\tree$] (q3) [right= 0.8cm of t2] {};
        
        \draw[->] (q0) edge node[weight] {$1$} (t0);
        \draw[->] (t0) edge node[weight] {$1$} (q1);
        \draw[->] (q1) edge node[weight] {$1$} (t1);
        \draw[->] (t1) edge node[weight] {$1$} (q2);
        \draw[->] (q2) edge node[weight] {$1$} (t2);
        \draw[->] (t2) edge node[weight] {$1$} (q3);
    \end{tikzpicture}
\end{description}

\noindent The initial marking $M_0$ is given by placing one token on each place $p_{?\varphi}$ for which $\nexists p \in P. (P, t_\varphi)\in F \wedge (t_\varphi, p_{?\varphi}) \in F$.

\begin{example}
Consider the TBT  $\Parallel_1([\F a, \Sequence(\F b, \G b)])$ and its corresponding Petri net given in Figure \ref{fig:tbtpetrinet}.
The colored regions indicate how recursively components of the TBT translation are interconnected, so that a token in place $p_{\Parallel}$ represents a satisfying TBT specification. 
\end{example}

\begin{figure}[h!]
    \centering
    \begin{tikzpicture}[node distance=1cm and 2cm,>=stealth',bend angle=45,thick]
        \node [place,tokens=1,label=below:$p_{?\F a}$] (q0) {};
        \node [transitionV,label=above:$t_{\F a}$] (t0) [right= 0.8cm of q0] {};
        \node [place,label=below:$p_{\F a}$] (q1) [right= 0.8cm of t0] {}; 
        \node [transitionV,label=above:$t_{\neg\F a}$] (t1) [above= 0.3cm of q0] {};
        \draw[->] (q0) edge node[weight] {$1$} (t0);
        \draw[->] (t0) edge node[weight] {$1$} (q1);
        \draw[->,bend left=70] (q0) edge node[weight] {$1$} (t1);
        \draw[->,bend left=70] (t1) edge node[weight] {$1$} (q0);
        
        \node [place,tokens=1,label=below:$p_{?\F b}$] (q2) [below= 2.4cm of q0] {};
        \node [transitionV,label=below:$t_{\F b}$] (t2) [right= 0.8cm of q2] {};
        \node [place,label=below:$p_{\F b}$] (q3) [right= 0.8cm of t2] {}; 
        \node [transitionV,label=above right:$t_{\neg\F b}$] (t3) [above= 0.3cm of q2] {};
        \draw[->] (q2) edge node[weight] {$1$} (t2);
        \draw[->] (t2) edge node[weight] {$1$} (q3);
        \draw[->,bend left=70] (q2) edge node[weight] {$1$} (t3);
        \draw[->,bend left=70] (t3) edge node[weight] {$1$} (q2);
        
        \node [transitionV,label=below:$t_1$] (t6) [right= 0.5cm of q3] {};
        
        \node [place,label=below:$p_{?\G b}$] (q4) [right= 0.7cm of t6] {};
        \node [transitionV,label=below:$t_{\G b}$] (t4) [right= 0.8cm of q4] {};
        \node [place,label=below:$p_{\G b}$] (q5) [right= 0.8cm of t4] {}; 
        \node [transitionV,label=above right:$t_{\neg\G b}$] (t5) [above= 0.3cm of q4] {};
        \draw[->] (q4) edge node[weight] {$1$} (t4);
        \draw[->] (t4) edge node[weight] {$1$} (q5);
        \draw[->,bend left=70] (q4) edge node[weight] {$1$} (t5);
        \draw[->,bend left=70] (t5) edge node[weight] {$1$} (q4);
        
        \draw[->] (q3) edge node[weight] {$1$} (t6);
        \draw[->] (t6) edge node[weight] {$1$} (q4);
        
        \node [transitionD,label=below right:$t_{\Sequence}$] (t7) [above = 1.7cm of q5] {};
        \node [place,label=left:$p_{\Sequence}$] (q6) [above left= 0.4cm and 1.4cm of t7] {}; 
        \draw[->] (q5) edge node[weight] {$1$} (t7);
        \draw[->] (t7) edge node[weight] {$1$} (q6);

        \node [transitionH] (t9) [above= 0.6cm of q6] {};
        \node [place,label=above left:$p_{\mathit{sync}}$] (q7) [above= 0.6cm of t9] {};
        \node [transitionV] (t8) [left= 0.5cm of q7] {};
        \node [place, tokens=1,label=right:$c_{\Parallel}$] (q8) [above= 0.3cm of q7] {}; 
        \node [transitionV, label=below:$t_{\Parallel}$] (t10) [right= 0.7cm of q7] {};
        \node [place, label=right:$p_{\Parallel}$] (q9) [right= 0.6cm of t10] {}; 
        
        \draw[->] (q6) edge node[weight] {$1$} (t9);
        \draw[->] (t9) edge node[weight] {$1$} (q7);
        \draw[->] (q1) edge node[weight] {$1$} (t8);
        \draw[->] (t8) edge node[weight] {$1$} (q7);
        \draw[->] (q7) edge node[weight] {$1$} (t10);
        \draw[->] (q8) edge node[weight] {$1$} (t10);
        \draw[->] (t10) edge node[weight] {$1$} (q9);
        
    \draw[draw=blue!50, fill opacity=0.3, fill=blue!20, rounded corners] (-0.8,-1) rectangle (2.8,2);
    \draw[draw=blue!50, fill opacity=0.3, fill=blue!20, rounded corners] (-0.8,-4) rectangle (2.8,-1.4);
    \draw[draw=blue!50, fill opacity=0.3, fill=blue!20, rounded corners] (3.5,-4) rectangle (7.6,-1.4);
    
    \draw[draw=orange!50, fill opacity=0.3, fill=orange!20, rounded corners] (-0.9,-1.3) -- (-0.9,-4.1) -- (7.7,-4.1) -- (7.7,-0.15) -- (3.3,-0.15) -- (3.3,-1.3) -- cycle;
    
    \draw[draw=violet!50, fill opacity=0.3, fill=violet!20, rounded corners] (7.7,-0.1) -- (3.3,-0.1) -- (3.3,2.7) -- (7.7,2.7) -- cycle;

    \end{tikzpicture}
    \caption{The colored regions show the mapping between the TBT specification $\Parallel_1([\F a, \Sequence(\F b, \G b)])$ and its Petri net translation: blue for leaf nodes, orange for sequence, and violet for parallel.}
    \label{fig:tbtpetrinet}
\end{figure}
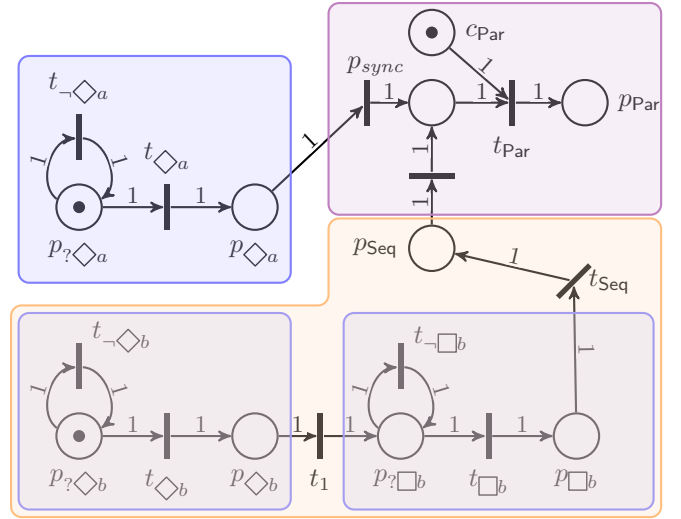

\subsection{Assigning Guards and Actions}
Next, we assign guards and actions to parts of the Petri net.
We begin by constructing the set $V$ of guard variables and the guard function $G$ for a TBT $\tree$. 
Let $A^{\varphi}$ denote the automaton constructed according to Def.\ \ref{def:ltl3} for an LTL$_3$ formula $\varphi$.
For a TBT $\tree$, let the set of automata be given by $A^\tree = \{ A^\varphi \mid \Leaf(\varphi) \in \tree \}$, i.e., each $\Leaf(\varphi)$ in $\tree$ has a corresponding automaton in $A^\tree$ where all automata receive data from the RL agent's execution and produce verdicts.
The set $V$ is given by $A^\tree$.
We add the following guards to each transition $t_{\varphi}$ and $t_{\neg\varphi}$: $G(t_{\varphi})=\true$ and $G(t_{\neg\varphi})=\false$.
Note that guards ensure that the constructed Petri nets are deterministic as they resolve the non-determinism in the leaf node construction.
The set $A$ consists of the following actions:
(1) $\mathit{start}(A^\varphi)$, which initializes $A^\varphi$ to its initial state where it resumes processing incoming inputs,  
(2) $\mathit{stop}(A^\varphi)$, which halts $A^\varphi$.  
For each transition $t$ with $(t, p_{?\varphi}) \in F$, we assign $\mathit{Act}(t)=\mathit{start(A^\varphi)}$.
Similarly, for each transition $t$ with $(t, p_{\varphi}) \in F$, we assign $\mathit{Act}(t)=\mathit{stop(A^\varphi)}$.
For simplicity, initially all $A^\varphi$ are started.
The guards ensure that tokens only propagate upon satisfaction through the net.
The actions manage the automata and allow what we refer to as \emph{backtracking}.
Backtracking happens when an automaton $A^{\varphi}$ returns $\false$, which means that no future event will result in satisfying the leaf node.
By leveraging TBT semantics, backtracking resets the failed leaf node and allows the agent to try again.
Later, we will see that backtracking results in penalizing the agent's rewards.
Let $\sigma$ be the execution trace of the RL agent with a TBT specification $\tree$.
We consider the case when a token reaches the place $p_\tree$ thereby indicating a successfull TBT RPN run.

\begin{theorem}
$M(p_\tree)=1 \nLeftrightarrow \sigma \models \tree$
\end{theorem}
\begin{proof}
For $\Leaf(\G a)$ and a trace $\sigma:= (a,a,\neg a, a, a)$, we have $\sigma \not\models \Leaf(\G a)$, but since backtracking resets the automaton $A^{\G a}$ after observing $\neg a$, a token reaches $p_{\tree}$.
\end{proof}

If we consider only leaf formulas with a corresponding automata that do not produce $\false$-verdicts, e.g., a co-safety property such as $\F a$, then the theorem holds.
Let $\tree'$ be a TBT specification with no automaton in $A^\tree$ that yields $\false$.

\begin{theorem}
$M(p_\tree')=1 \Leftrightarrow \sigma \models \tree'$
\label{thm:nofalse}
\end{theorem}
\begin{proof}
The proof is by structural induction on $\tree$.
A trace satisfies $\Leaf(\varphi)$ if and only if the guard of transition $t_{\varphi}$ is $\true$ and the token reaches $p_{\varphi}$, i.e., $M(p_{\varphi})=1$.
By induction hypothesis, the recursive constructions (depicted by dotted circles) contain a token if and only if the corresponding subtree is satisfied.
For a $\Fallback$-node, the control place $c_\tree$ ensures only a single token reaches $p_{sync}$ and eventually $p_\tree$.
This corresponds to its TBT semantics where at least one of the subtrees must be satisfied.
For a $\Parallel_M$-node, the place $p_{sync}$ gathers $M$ tokens until $t_\tree$ is enabled, i.e., it satisfies its TBT semantics that at least $M$ subtrees are satisfied.
The control place $c_\tree$ ensures that the node can only be satisfied once.
For a $\Sequence$-node, once $\tree_1$ is satisfied the token moves on to the next place, finally reaching $p_\tree$.
\end{proof}

Note that Theorem \ref{thm:nofalse} is generalized to RL agent's executions where the Petri net never fires a transition $t_{\neg\varphi}$, i.e., the TBT specification is satisfied when no transitions 
$t_{\neg\varphi}$ is enabled and a dead marking is reached.

\subsection{Assigning Rewards}
Let $\tree$ be a TBT, and let $\mathit{Fn}=(\mathit{fn}_{\Leaf}, \mathit{fn}_{\Fallback}, \mathit{fn}_{\Parallel_M}, \mathit{fn}_{\Sequence})$ be a collection of assignment functions, one for each node type.
Each assignment function maps transitions to rewards.
For a reward $r \in \mathbb{R}$, the reward distribution function $d_\tree^{\mathit{Fn}}(r)$ is defined by recursion on $\tree$ and applies the assignment functions depending on the node type.

\begin{example}
An assignment function $\mathit{fn}_{\Leaf(\varphi)}= \{R(t_\varphi)=r, R(t_{\neg\varphi})= -0.1 \cdot r\}$ yields a positive reward $r$ on success and reduced negative rewards upon failure.
A function $\mathit{fn}_{\Fallback}(\tree_1, \dots, \tree_n, r) = \{ R(t_{\tree_1}) = 0, \dots, R(t_{\tree_n}) = 0 \} \cup {\bigcup{}}_{i=1}^n  d_{\tree_i}(r)$ and $\mathit{fn}_{\Sequence}(\tree_1, \dots, \tree_n, r) = \{ R(t_{\tree_1}) = 0, \dots, R(t_{\tree_{n}}) = 0 \} \cup \bigcup_{i=1}^{n} d_{\tree_i}(\frac{r}{n}\cdot (i-1))$ do not reward transitions $\tree_i$ of the node but forward the reward to the subnodes instead. 
While $\mathit{fn}_{\Fallback}$ forwards the full reward $r$ to its subtrees, $\mathit{fn}_{\Sequence}$ assign smaller rewards to earlier subtrees and larger rewards to later subtrees.
\label{ex:assign}
\end{example}

Since branching with TBT nodes may alter the reward distribution, the TBT structure can be exploited to control the reward assignment.
Using $\mathit{fn}_\Sequence$ from Example \ref{ex:assign}, the TBT $\Sequence(\F a, \F, b, \F c)$ assigns equal rewards to all leaf nodes.
In contrast, $\Sequence([ \Sequence([ \F a, \F b ]), \F c ] )$ assigns a higher reward to the last leaf, while  $\Sequence([\F a, \Sequence([\F b, \F c])])$ assigns a higher reward to the first leaf.

\subsection{Embedding in Markov Decision Process} \label{sec:mdprpn}
We now define how the reward Petri net is embedded within the Markov Decision Process for RL.

\begin{definition}
A Markov decision process with a reward Petri net (MDPRPN) is a tuple $MP=(M, RN)$ where $M$ is a MDP $\mathcal{M} = (S, A, P, R, \rho_0)$ and $RN$ is a reward Petri net $\mathcal{RN} = (P, T, F, M_0, R, A, \mathit{Act}, V, G)$.
\label{def:mdprpn}
\end{definition}

In general, the MDPRPN state space is comprised of MDP states and RPN markings.
At every environment step, the marking is updated after the RPN fires transitions, based on the verdicts of each started automaton in $A^\tree$.
This yields the new state and reward $R(s,a,s')$ in $M$ is $\sum_{t \in T: t \text{ is firing }} R(t)$.


\section{Experiments}\label{sec:experiments}

In this section, we first introduce the environment framework and algorithm details before applying TBT-derived RPNs to increasingly complex problems. We showcase the adaption of TBTs with each step in complexity and how different reward assignments along the respective RPNs influence learning performance. Further details (including visualizations) of the examined environments, algorithm hyperparameters, and reward distributions are provided in the appendix.



\subsection{Implementation} \label{subsec:implementation}

The MiniGrid framework is a popular implementation of procedurally generated gridworlds for RL \cite{minigrid}.
Figure \ref{fig:minigrid} depicts a typical MiniGrid environment with the task of locating and picking up a specific goal object (blue circle). This requires the agent (red triangle) to open boxes (gray) to find the red key. Because the agent can only hold one item at a time, blockers (green circle) must be moved before the agent can open the locked red door with key in hand. For the specific environments discussed in this section, the observation space is limited to the symbolic representation of the agent's field-of-view (FOV). We extend the canonical observation for MDPRPN as described in Section \ref{sec:mdprpn}. By default, MiniGrid uses a sparse reward formulation with a discount for episode length: $r=1-0.9(n_\mathit{steps}/n_\mathit{max\_steps})$ where $n_\mathit{steps}$ is the number of steps until successful termination and $n_\mathit{max\_steps}$ is the step limit for episode truncation.

\begin{figure}[hbt]
    \centering
    \includegraphics[width=0.7\linewidth]{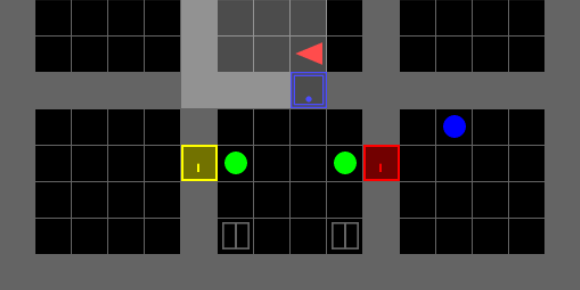}
    \caption{Excerpt of a \texttt{MiniGrid-ObstructedMaze-1Q} generation where the agent's (red triangle) FOV is depicted in light grey.}
    \label{fig:minigrid}
\end{figure}


We apply the PPO algorithm introduced by \cite{ppo} as implemented by Stable Baselines3 \cite{stable-baselines3} in addition to the related RL Baselines3 Zoo \cite{rl-zoo3}. Hyperparameters are included in the supplementary material. For comparability, functionally identical SB3–MiniGrid integrations are used in both TBT and non-TBT experiments.

The TBT is implemented in Python and uses pytransitions \cite{pytransitions} and py-metric-temporal-logic \cite{pyMTL} for its internal logic. LTL$_3$ to automaton translation was performed using LTL$_\mathrm{f}$2DFA \cite{ltlf2dfa} and LamaConv \cite{lamaconv}. Atomic propositions and corresponding TBT transitions are evaluated every step, taking the default environment observation as input. Therefore, no information beyond the agent's observation (i.e. positions) are encoded by the specification, making the specification robust to procedural generation.
The output of the derived RPN is complemented by a constant step reward $r_{step} = -0.9/n_\mathit{max\_steps}$, which distributes the canonical episode length discount across each step. This disincentivizes policies that prefer inaction to avoid potential penalties associated with violations in the TBT specification. 

Experiments were run on workstations generously made available by WüSpace e.V.\footnote{WüSpace is a student aerospace organization. “WüSpace e.V. – Space to Explore,” \url{https://www.wuespace.de}} equipped with a Ryzen 9 7950X 32-core CPU and NVidia GeForce RTX 4090 GPU. Training takes approximately 3.5 hours per $10^6$ timesteps per used processor core.

\subsection{TBT Specification}

In the following, we briefly introduce a series of increasingly difficult MiniGrid environments and the TBTs used to solve them.

\newlength{\tbtwidth}
\setlength{\tbtwidth}{0.798\linewidth} 
\begin{figure}[htpb]
\vspace{-5pt}
    \centering
    \begin{subfigure}
        \centering
        \setlength{\abovecaptionskip}{0pt}
        \includegraphics[width=0.5659164\tbtwidth]{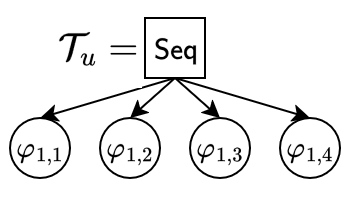}
        \caption{TBT for \texttt{Unlock} environment. Consult Table \ref{tab:leafs} for leaf node formulas.}
        \label{fig:tbt-u}
    \end{subfigure}
    
    \begin{subfigure}
        \centering
        \setlength{\abovecaptionskip}{0pt}
        \includegraphics[width=0.71061093\tbtwidth]{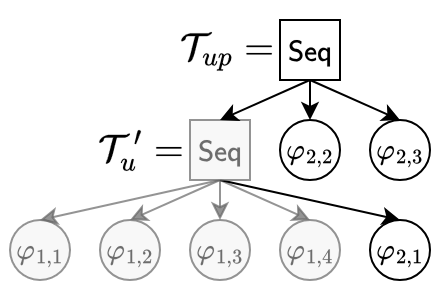}
        \caption{TBT for \texttt{UnlockPickup} environment. Nodes adopted from $\tree_u$ shown in gray (compare with Figure \ref{fig:tbt-u}.)}
        \label{fig:tbt-up}
    \end{subfigure}
    
    \begin{subfigure}
        \centering
        \setlength{\abovecaptionskip}{0pt}
        \includegraphics[width=\tbtwidth]{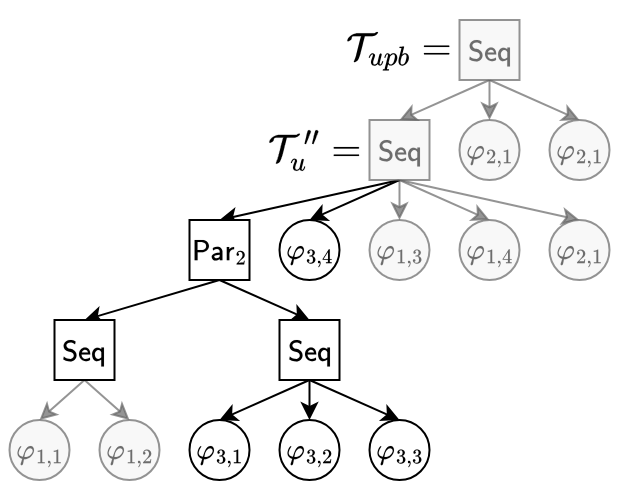}
        \caption{TBT for \texttt{BlockedUnlockPickup} environment. Nodes adopted from $\tree_{up}$ shown in gray (compare with Figure \ref{fig:tbt-up}.)}
        \label{fig:tbt-upb}
    \end{subfigure}
    
    \begin{subfigure}
        \centering
        \setlength{\abovecaptionskip}{0pt}
        \includegraphics[width=0.88102894\tbtwidth]{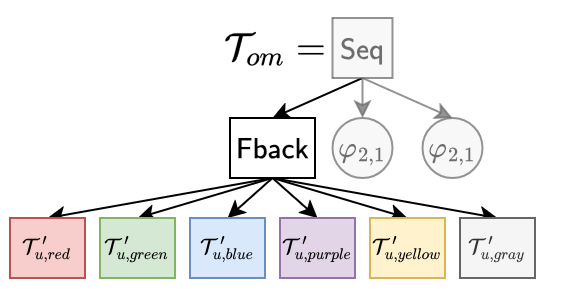}
        \caption{TBT for \texttt{ObstructedMaze-2Dlh} environment. Uses $\tree_{u}'$ (see Figure \ref{fig:tbt-up}) and parametrized $\Fallback\langle x\in\{red,\allowbreak  green,\allowbreak  blue,\allowbreak  purple,\allowbreak  yellow,\allowbreak  gray\}\rangle(\tree_{u,x}')$.}
        \label{fig:tbt-om}
    \end{subfigure}
    
    \begin{subfigure}
        \centering
        \setlength{\abovecaptionskip}{0pt}
        \includegraphics[width=0.88102894\tbtwidth]{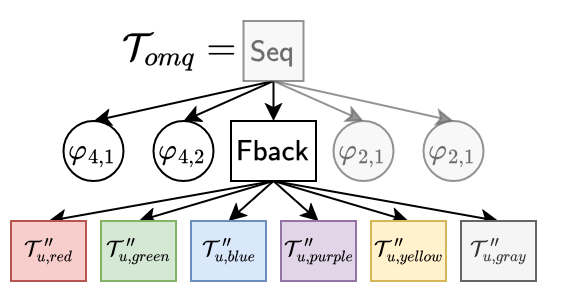}
        \caption{TBT for \texttt{ObstructedMaze-1Q} environment. Uses $\tree_{u}''$ (see Figure \ref{fig:tbt-upb}) and parametrized $\Fallback\langle x\in\{red,\allowbreak  green,\allowbreak  blue,\allowbreak  purple,\allowbreak  yellow,\allowbreak  gray\}\rangle(\tree_{u,x}'')$.}
        \label{fig:tbt-omq}
    \end{subfigure}
\end{figure} 

\begin{table}[htpb]
    \setlength{\belowcaptionskip}{0pt}
    \setlength{\abovecaptionskip}{5pt}
    \caption{Leaf Node Definitions. Note that for conciseness $\mathit{Door}$ and $\mathit{Key}$ are used interchangeably with $\mathit{xDoor}$ and $\mathit{xKey}$. The former describes objects of any color, while the latter applies when $x$ is directly specified, for example using a $\Fallback$ as in $\tree_{om}$ (Figure \ref{fig:tbt-om}).}
    \label{tab:leafs}
    \centering
    \footnotesize 
    \setlength\tabcolsep{2pt}
    \begin{tabular*}{\columnwidth}{@{\extracolsep{\fill}} |c|l|c|l|}
    \hline
    $\mathbf{\varphi}$ & \textbf{Formula} & $\mathbf{\varphi}$ & \textbf{Formula} \\
    \hline
    $\varphi_{1,1}$ & $\F \mathit{seeKey}$ & $\varphi_{1,2}$ & $\F \mathit{holdKey}$\\
    $\varphi_{1,3}$ & $\mathit{holdKey} \Un \mathit{faceDoor}$ & $\varphi_{1,4}$ & $\mathit{faceOpenDoor}$\\
    $\varphi_{2,1}$ & $\F \mathit{seeGoal}$ & $\varphi_{2,3}$ & $\mathit{holdGoal}$\\
    $\varphi_{2,2}$ & \multicolumn{3}{l|}{$\mathit{seeGoal}\ \Un\ (\mathit{faceGoal} \land \mathit{holdNothing})$}\\
    $\varphi_{3,1}$ & \multicolumn{3}{l|}{$\F \mathit{seeDoor}$}\\
    $\varphi_{3,2}$ & \multicolumn{3}{l|}{$\F (\mathit{seeDoor} \land \mathit{holdBlocker})$}\\
    $\varphi_{3,3}$ & \multicolumn{3}{l|}{$\F (\mathit{seeDoor} \land \neg\mathit{holdBlocker} \land \neg\mathit{blockedDoor})$}\\
    $\varphi_{3,4}$ & \multicolumn{3}{l|}{$\F (\mathit{seeDoor} \land \mathit{holdKey} \land \neg\mathit{blockedDoor})$}\\
    $\varphi_{4,1}$ & $\F \mathit{seeUnlocked}$ & $\varphi_{4,2}$ & $\mathit{seeUnlocked}\ \Un\ \mathit{faceOpenDoor}$\\
    \hline
    \end{tabular*}
\end{table}


\begin{figure*}
    \centering
    \setlength{\belowcaptionskip}{0pt}
    \setlength{\abovecaptionskip}{0pt}
    \includegraphics[width=\linewidth]{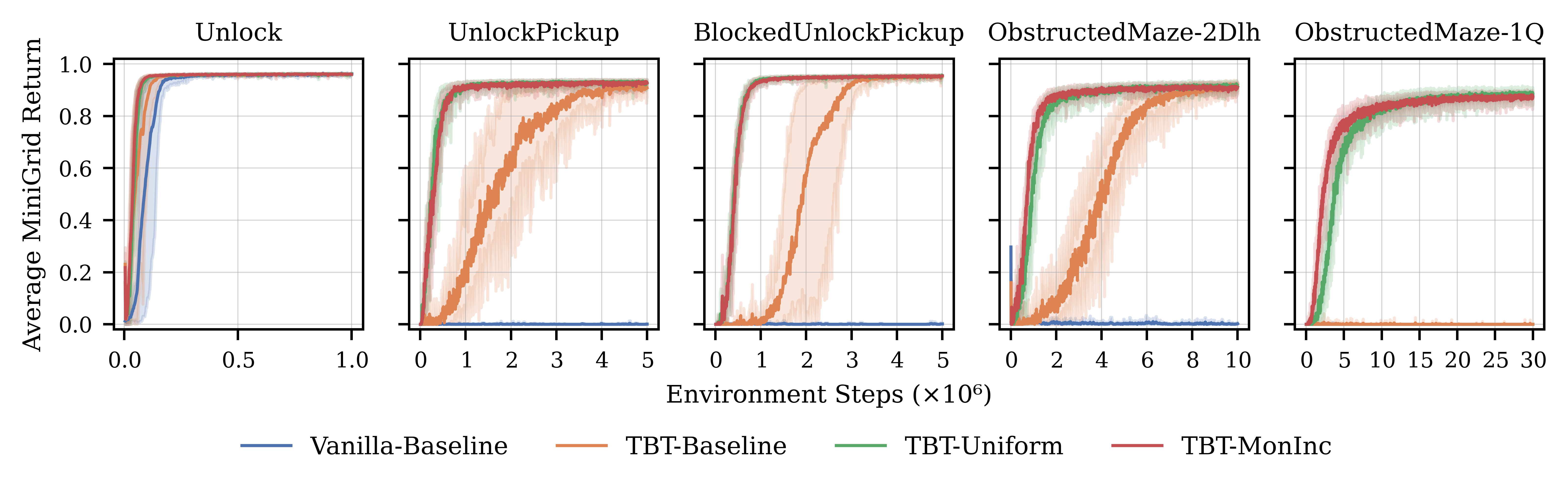}
    \caption{Comparison of training performance on a series of increasingly hard MiniGrid environments.}
    \label{fig:exp-res}
\end{figure*}

\begin{description}

\item[\texttt{Unlock}] serves as the starting point for our analysis. The agent is tasked with moving to the key, picking it up, taking it to the locked door and finally unlocking it. These sequential steps are reflected in $\tree_u$ (Figure \ref{fig:tbt-u}).

\item[\texttt{UnlockPickup}] extends the assignment by additionally requiring the agent to enter the second room, discard the key, and instead navigate towards and pick up the goal object. We modify the previous TBT into $\tree_{up}$ (Figure \ref{fig:tbt-up}) with a nested Sequence and designate the inner Sequence $\tree_u'$ in anticipation of later environments.

\item[\texttt{BlockedUnlockPickup}] requires the agent to pick up and remove a blocking obstacle while preparing to unlock the door. A Parallel operator is applied as can be seen in Figure \ref{fig:tbt-upb} to permit flexible reactions to any procedural generation, instead of prescribing a fixed order to prioritize either key or obstacle before the other. Note again the designation of the inner Sequence as $\tree_u''$.

\item[\texttt{ObstructedMaze-2Dlh}] belongs to environments of the \texttt{ObstructedMaze} class, which expand the navigable environment and increase the number of locked doors hindering access to the goal. This class is considered a hard exploration domain \cite{intrinsic} for which no successful PPO configuration was found so far \cite{rl-zoo3}. In \texttt{ObstructedMaze-2Dlh}, the agent starts between three doors, (no obstacles for now,) and two boxes, containing keys for the two locked doors. We adapt $\tree_{up}$ using a color-switching Fallback around $\tree_u'$ so that the agent receives reward for opening every locked door until Fallback satisfaction when the goal is found.

\item[\texttt{ObstructedMaze-1Q}] (see Figure \ref{fig:minigrid}) reintroduces the blocking obstacles to the locked doors and starts the agent one room and one unlocked door further away from its goal. $\tree_{upb}$ proved effective at opening obstructed doors so we repeat the latest adaption to it and wrap $\tree_u''$ with a Fallback. Additionally, we extend the root Sequence to guide the agent through the first door.

\end{description}

\subsection{Results}

Next, we evaluate and compare four approaches on the introduced environments and TBTs:
(I) \textsc{Vanilla-Baseline} using standard MiniGrid as described in Section \ref{subsec:implementation};
(II) \textsc{TBT-Baseline}, where the TBT is included in the observation space and the unit reward is issued only upon satisfaction of the root node;
(III) \textsc{TBT-Uniform} variant, in which the unit reward is distributed evenly between all TBT nodes through a shared assignment function;
and (IV) \textsc{TBT-MonInc} variant, where the unit reward is distributed only to leaf nodes such that along the RPN, leaf node satisfaction rewards ($t_\varphi$) are strictly monotonically increasing. The intuition behind this is the prioritization of later subgoals and by extension task completion.
Detailed descriptions of the specific reward distributions and resulting reward assignments used in the experiments can be found in Appendix C and the experiment configurations in the supplementary material.
All TBT-based reward schemes additionally incorporate a failure penalty of the form $R(t_{\neg \tree}) = -\alpha R(t_\tree)$, with $\alpha = 0.01$ selected via parameter sweep. 
 
Figure \ref{fig:exp-res} shows training performance averaged over 5 runs with minimum and maximum values shown in reduced opacity. Average MiniGrid Return refers to episode rewards emitted by the default MiniGrid reward definition. When the agent receives RPN output as a reward, default MiniGrid rewards are recorded separately for comparability. Our results show that involving TBTs generally improves sample efficiency compared to the vanilla baseline and specifically enables learning where vanilla RL fails. 
\textsc{TBT-Baseline} demonstrates that using a sparse MDPRPN can have this effect, albeit with high variance and expectedly inferior performance compared to the dense reward functions (III) and (IV). Further, \textsc{TBT-MonInc} exhibits a slight advantage in learning speed for the \texttt{ObstructedMaze} environments.
Importantly, TBT-based approaches permit scaling with environment complexity to efficiently solve hard exploration problems.
Note that when normalizing training results with episode length, a stable policy can be achieved using the \textsc{TBT-MonInc} within an approximate range of $3\cdot10^3$ to $13\cdot10^3$ episodes for the environments beyond the trivial \texttt{Unlock} environment, (which only takes approximately $500$ episodes to train.) Tensorboard files with additional training summaries will be made available.
\section{Conclusion}\label{sec:conclusion}
We studied the problem of generating rewards for long-horizon tasks in reinforcement learning using Temporal Behavior Trees. 
TBTs can easily be adapted to different environments leveraging their high expressivity and modularity.
We introduced Reward Petri nets (RPN) and showed how TBTs can be naturally translated into RPNs. RPNs extend Petri nets by guard predicates, rewards, and actions upon firing transitions.
We then introduced reward distribution functions over RPNs, which propagate rewards through the behavior specification. We demonstrated that these facilitate sample efficient learning of hard exploration problems.

In the future, we plan to adapt value iteration for automatic reward shaping to RPN \cite{icarte2019ltl}. 
Further, TBTs are an ideal formalism for expressing multi-agent RL tasks, which we plan to exploit.
RPNs feature considerable compatibility as (Reward) Petri net were originally invented to handle distributed processes with concurrent tasks. 
We also plan to extent existing TBT monitoring approaches to act as shields during learning. 





\bibliographystyle{named}
\nocite{hpo,rl-zoo3,ppo-hyperparams-aux,rs}
\bibliography{references}


\end{document}